\newif\ifICML

\ICMLfalse

\documentclass{article}
\usepackage{hyperref}
\usepackage{url}
\usepackage{amsmath}
\usepackage{amsthm,amsfonts,amssymb,epsfig,color,float,graphicx,verbatim}
\usepackage{algorithm,algorithmic}
\usepackage{times}
\usepackage{graphicx} 
\usepackage{subfigure}
\usepackage{amsfonts}
\usepackage{natbib}

\usepackage{algorithm}
\usepackage{algorithmic}

\usepackage{hyperref}

\ifICML
\usepackage{icml2014}
\else
\usepackage{fullpage}
\fi

\newcommand{\reals}{\mathbb{R}}

\newcommand{\ww}{\mathbf{w}}
\newcommand{\xx}{\mathbf{x}}
\newcommand{\vv}{\mathbf{v}}

\newcommand{\uu}{\mathbf{u}}
\newcommand{\oo}{\mathbf{o}}

\renewcommand{\ss}{\mathbf{s}}

\newcommand{\im}{\mathrm{Im}}

\newcommand{\rank}{\mathrm{rank}}
\newcommand{\epg}{$\epsilon$-good }
\newcommand{\secref}[1]{Section \ref{#1}}
\newcommand{\defref}[1]{Definition \ref{#1}}
\newcommand{\coref}[1]{Corollary \ref{#1}}
\newcommand{\claimref}[1]{Claim \ref{#1}}

\newtheorem{lemma}{Lemma}
\newtheorem{theorem}{Theorem}
\newtheorem{corollary}{Corollary}

\newtheorem{definition} {Definition}
\newtheorem{claim} {Claim}

\newtheorem{assumption}{Assumption}

\newcommand{\probref}[1]{Problem \ref{#1}}

\usepackage{color}
\newcommand{\sampx}[1]{\xx^{(#1)}}
\newcommand{\sampo}[1]{\oo^{(#1)}}
\newcommand{\sampox}[1]{\xx_{\oo^{#1}}^{#1}}

\newcommand{\asref}[1]{Assumption \ref{#1}}
\newcommand{\thmref}[1]{Theorem \ref{#1}}
\newcommand{\algref}[1]{Alg. \ref{#1}}
\newcommand{\lemref}[1]{Lemma \ref{#1}}
\renewcommand{\eqref}[1]{Eq. \ref{#1}}

\newcommand{\ignore}[1]{}

\newcommand{\obj}[2]{ & \underset{#1}{\text{minimize}}& & #2 \\}
\newcommand{\st}[1]{ &  \text{subject to:}& & #1}

\newenvironment{program}[1]{\begin{equation*}\tag{#1}\begin{aligned}}{.\end{aligned}\end{equation*}}

\newenvironment{Ouralgorithm}[1][\  ] %
{
\rm
\begin{tabbing}
.\=...\=...\=...\=...\=  \+ \kill
} %
{\end{tabbing}
}
\floatstyle{ruled}
\newfloat{Algorithm}{thp}{alg}
{
\begin{minipage}{1.0\linewidth}
\begin{Ouralgorithm} %
} { \end{Ouralgorithm} \end{minipage} }

\ifICML
\icmltitlerunning{Classification with Low-Rank and Missing Data}
\else
\title{Classification with Low-Rank and Missing Data}
\fi
\begin{document}

\ifICML
\onecolumn
\icmltitle{Classification with Low Rank and Missing Data}
\else
\title{Classification with Low Rank and Missing Data}
\date{}
\fi

\ifICML
\icmlauthor{Elad Hazan \and  Roi Livni \and  Yishay Mansour}{email@yourdomain.edu}
\icmladdress{Your Fantastic Institute,
            314159 Pi St., Palo Alto, CA 94306 USA}
\icmlauthor{Your CoAuthor's Name}{email@coauthordomain.edu}
\icmladdress{Their Fantastic Institute,
            27182 Exp St., Toronto, ON M6H 2T1 CANADA}
\else
\author{Elad Hazan \thanks{Princeton University and Microsoft Research} 
\and  Roi Livni \thanks{Hebrew University and Microsoft Research} 
 \and  Yishay Mansour  \thanks{Tel Aviv University and Microsoft Research}  }
\maketitle
\fi

\ifICML
\icmlkeywords{boring formatting information, machine learning, ICML}
\fi

\vskip 0.3in

\begin{abstract}
We consider classification and regression tasks where we have missing data and assume that the (clean) data resides in
a low rank subspace. Finding a hidden subspace is known to be computationally hard.
Nevertheless, using a non-proper formulation we give an efficient agnostic algorithm that classifies as good as the best linear classifier coupled with the best low-dimensional subspace in which the data resides. A direct implication is that our algorithm can linearly (and non-linearly through kernels) classify provably as well as the best classifier that has access to the full data. 
\end{abstract}

\section{Introduction}\label{sec:intro}


The importance of handling correctly missing data is a fundamental
and classical challenge in machine learning.
 There are many reasons why data might be missing.
 For example, consider the medical domain, some data might be missing because certain
 procedures were not performed on a given patient,
 other data might be missing because the patient choose not to disclose them,
 and even some data might be missing due to malfunction of certain equipment.
While it is definitely much better to have always complete and accurate data,
this utopian desire is not the reality many times. For this reason we need to utilize
the available data even if some of it is missing.


Another, very different motivation for missing data are recommendations.
%
%
For example, a movie recommendations dataset  might have users opinions on certain movies, which
is the case, for example, in the Netflix motion picture dataset.
Clearly, no user has seen or reviewed all movies, or even close to it.
In this respect recommendation data is an extreme case: the vast majority is usually missing
(i.e., it is sparse to the extreme).
%

Many times we can solve the missing data problem since the data resides on a lower dimension manifold.
In the above examples, if there are prototypical users (or patients) and any user is a mixture
of the prototypical users, then this implicitly
suggests that the data is {\em low rank}. Another way to formalize this assumption is to consider the data in a matrix form, say, the users are rows and movies are columns, then our assumption is that the true
complete matrix has a low rank.


%


Our starting point is to consider the low rank assumption, but to avoid any explicit matrix completion, and
instead directly dive in to the classification problem. At the end of the introduction we show that matrix completion is neither sufficient and/or necessary.

We consider perhaps the most fundamental data analysis technique of the machine learning toolkit: linear (and kernel) classification, as applied to data where some (or even most) of the attributes in an example might be missing.  {\bf Our main result is an efficient algorithm for linear and kernel classification that performs as well as the best classifier that has access to all data}, under low rank assumption with natural non-degeneracy conditions.

We stress that our result is worst case, we do not assume that the missing data follows any probabilistic rule other than the underlying matrix having low rank. This is a clear contrast to most existing matrix completion algorithms.
We also cast our results in a distributional setting, showing that the classification
error that we achieve is close to the best classification using the subspace of the examples (and with no missing data). Notably, many variants of the  problem of finding a hidden subspace are computationally hard (see e.g. \cite{BerRig13}), yet as we show, learning a linear classifier on a hidden subspace is non-properly learnable.

At a high level, we assume that a sample is a triplet $(\xx,\oo,y)$, where $\xx\in \mathbb{R}^d$ is the complete example, $\oo\subset \{1,\ldots,d\}$ is the set of observable attributes and $y\in\mathcal{Y}$ is the label.
The learner observes only $(\xx_o, y)$, where $\xx_o$ omits any attribute not in $\oo$.
Our goal is given a sample $S=\{(\xx_o^{(i)},y^{(i)})\}_{i=1}^m$ to output
a classifier $h_S$ such that w.h.p.:
\begin{equation*}
 \mathop\mathbb{E}\left[\ell (h_S(\xx_\oo),y)\right] \le \min_{\stackrel{\ww\in \mathbb{R}^d}{\|\ww\|\le 1}}\mathop\mathbb{E}\left[ \ell ( \ww\cdot \xx,y)\right] +\epsilon ,
\end{equation*}
where $\ell$ is the loss function. Namely, we like our classifier $h_S$ to compete with the best
linear classifier for the completely observable data.

Our main result is achieving this task (under mild regularity conditions)
using a computationally efficient algorithm for any convex Lipschitz-bounded loss function.
Our basic result requires a sample size which is quasi-polynomial, but we complement it
with a kernel construction which can guarantee efficient learning under appropriate large margin assumptions. Our kernel depends only on the intersection of observable values of two inputs,
and is efficiently computable.
(We give a more detailed overview of our main results in Section \ref{sec:prob}.)

\ifICML
We complement our theoretical contributions with experimental findings that show superior classification performance both on synthetic data and on publicly-available recommendation data.
\else
Preliminary experimental evidence indicates our  theoretical contributions lead  to  promising classification performance both on synthetic data and on publicly-available recommendation data. This will be detailed in the full version of this paper. 
\fi



\paragraph{Previous work.}

Classification with missing data is
a well studied subject in statistics with numerous books and papers devoted to its study, (see, e.g., \cite{LittleRubin}).
The statistical treatment of missing data is broad, and to a fairly large extent assumes parametric models both for the data
generating process as well as the process that creates the missing data.
One of the most popular models for the missing data process is {\em Missing Completely at Random (MCAR)} where the missing attributes
are selected independently from the values.

We outline a few of the main approaches handling missing data in the statistics literature.
The simplest method is simply to discard records with missing data, even this assumes independence between the examples with missing values
and their labels. In order to estimate simple statistics, such as the expected value of an attribute, one can use importance sampling methods,
where the probability of an attribute being missing can depend on it value (e.g., using the Horvitz-Thompson estimator \cite{Horvitz-Thompson}).
A large body of techniques is devoted to {\em imputation} procedures which complete the missing data. This can be done by replacing
a missing attribute by its mean (mean imputation), or using a regression based on the observed value (regression imputation), or sampling the
other examples to complete the missing value (hot deck). \footnote{We remark that our model implicitly includes mean-imputation or $0$-imputation method
and therefore will always outperform them.}
The imputation
 methodologies share a similar goal as matrix completion, namely reduce the problem to one with complete data, however their methodologies and
motivating scenarios are very different. Finally, one can build a complete Bayesian model for both the observed and unobserved data and use it to perform
inference. As with almost any Bayesian methodology, its success depends largely on selecting the right model and prior,
this is even ignoring the computational issues which make inference in many of those models computationally intractable.

%

In the machine learning community, missing data was considered in the framework of limited attribute observability
\cite{ben1998learning} and its many refinements  \cite{DekelSL2010,cesa2010efficient,cesa2011online,HazanK12a}.
However, to the best of our knowledge, the low-rank property is not captured by previous work, nor is the extreme  amount of missing data.
More importantly, much of the research is focused on selecting {\em which attributes to observe} or on missing attributes at test or train time (see also \cite{eban2014discrete,globerson2006nightmare}). In our case the learner has no control
which attributes are observable in an example and the domain is fixed. The latter case is captured in the work of \cite{Chechik:2008:MCD:1390681.1390682}, who rescale inner-products according to the amount of missing data. Their method, however, does not entail theoretical gaurantees on reconstruction in the worst case, and gives rise to non-convex programs.

A natural and intuitive methodology to follow is to treat the labels (both known and unknown)
as an additional column in the data matrix and complete the data using a matrix completion algorithm, thereby obtaining the classification.
Indeed, this  exactly was proposed by \cite{GoldbergZRXN10}.
Although this is a natural approach,
we show that completion is neither necessary nor sufficient for classification.
Furthermore, the techniques for provably completing a low rank matrix are only known under probabilistic models with restricted distributions
\cite{SrebroThesis,CandesR09,lee:practical,salakhutdinov:collaborative,ShamirS11}.
The only non-probabilistic matrix completion algorithm in the online learning setting we are aware of is \cite{HazanKS12},
which we were not able to use for our purposes.

\paragraph{Is matrix completion sufficient and/or necessary?} We demonstrate that classification with missing data is provably different from
that of matrix completion.
We start by considering a learner that tries to complete the missing entries in an unsupervised manner and then performs classification on the completed data,
this approach is close akin to imputation techniques, generative models and any other two step --
unsupervised/supervised algorithm. Our example shows that even under realizable assumptions, such an algorithm may fail.
We then proceed to analyze the approach previously mentioned -- to treat the labels as an additional column.

To see that unsupervised completion is insufficient for prediction, consider the  example in Figure \ref{comp1}:
the original data is represented by filled red and green dots and it is linearly separable.
Each data point will have one of its two coordinates missing (this can even be done at random.
In the figure the arrow from each instance points to the observed attribute.
However, the rank-one completion of projection onto the pink hyperplane is possible, and admits no separation.
The problem is clearly that the mapping to a low dimension is independent from the labels, and therefore we should not expect that
properties that depend on the labels, such as linear separability, will be maintained.

\begin{figure}[h!] \label{comp1}
	\begin{center}
		\includegraphics[width=1.5in]{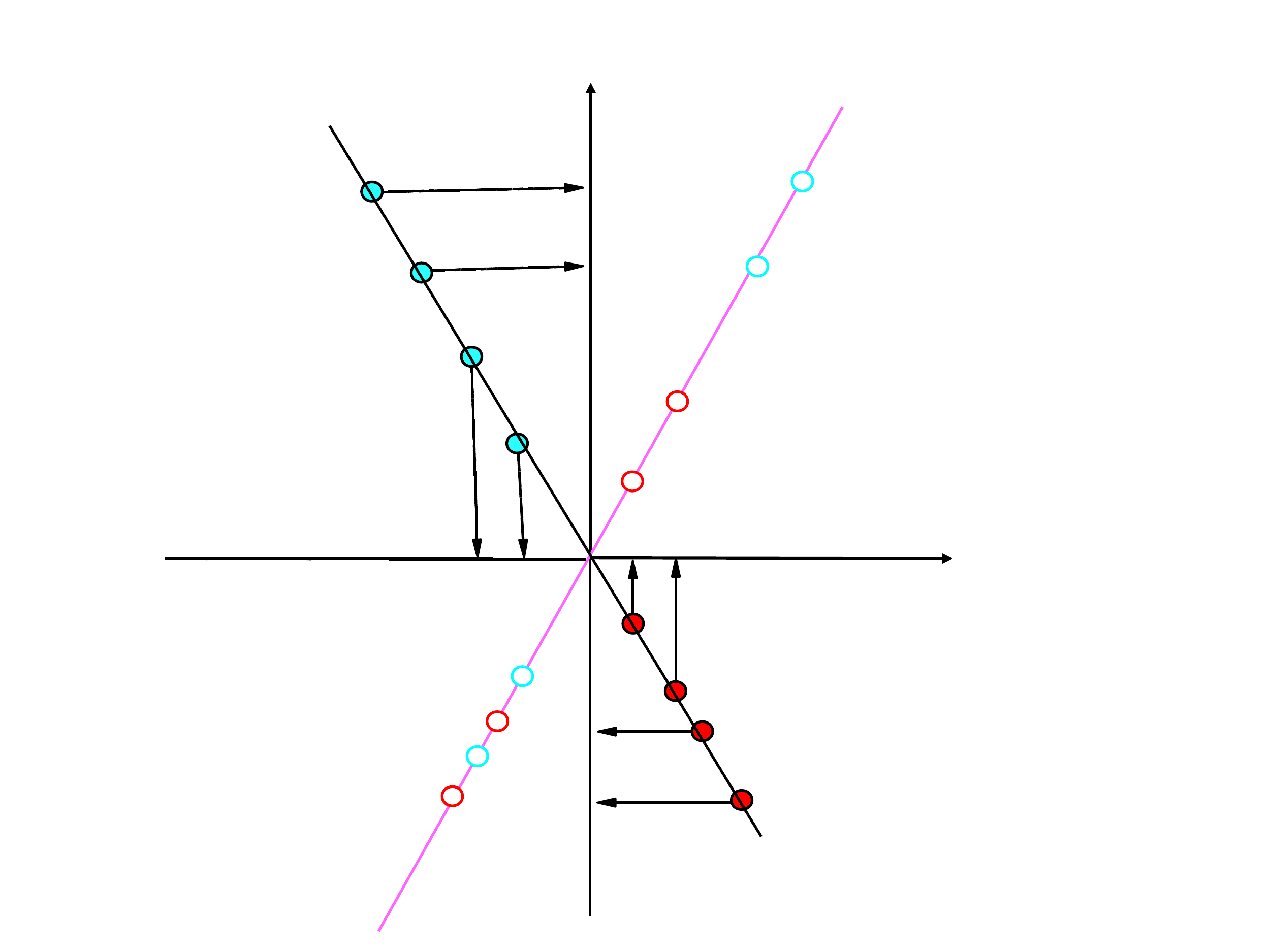}
	\end{center}
	\caption{Linearly separable data, for which certain completions make the data is non-separable.}
\end{figure}

Next, consider a learner that treats the labels as an additional column. \cite{GoldbergZRXN10} Considered the following problem:

\begin{program}{G}\label{prob:G}
\obj{Z}{\rank(Z)}
\st{ Z_{i,j}=\xx_{i,j}, & (i,j)\in \Omega}\;,
\end{program}
where $\Omega$ is the set of observed attributes (or observed labels for the corresponding columns).
Now assume that we always see one of the following examples: $[1, ~ *, ~ 1, ~* ]$, $[*, ~ -1, ~ *, ~ -1]$, or $[1, ~, -1, ~ 1, ~ -1]$.
The observed labels are respectively $1$,$-1$ and $1$. A typical data matrix with one test point might be of the form:

\begin{equation}\label{eq:M}M= \left[\begin{array}{rrrr||r}
1 & * & 1 & * & 1 \\
* & -1 &  * & -1 &-1 \\
1 & -1 & 1 & -1 & 1\\
1 & * & 1 & * & * \\
\end{array}\right]\end{equation}
First note that there is no $1$-rank completion of this matrix. On the other hand,
we will show that there is more than one $2$-rank completion each lead to a different classification of the test point.
The first possible completion is to complete odd columns to a constant one vector, and even column vectors to a constant $-1$ vector.
Then complete the labeling whichever way you choose.
Clearly there is no hope for this completion to lead to any meaningful result as the label vector is independent of the data columns.
On the other hand we may complete the first and last rows to a constant $1$ vector, and the second row to a constant $-1$ vector. All possible completions lead to an optimal solution w.r.t \probref{prob:G} but have different outcome w.r.t classification.
We stress that this is not a sample complexity issue. Even if we observe abundant amount of data, the completion task is still ill-posed.

Finally, matrix completion is also not necessary for prediction.
Consider movie recommendation dataset  with two separate populations, French and Chinese, where each population reviews a different set of movies.
Even if each population has a low rank, performing successful matrix completion, in this case, is impossible (and intuitively it does not make sense
in such a setting).
%
%
However, linear classification in this case is possible via a single linear classifier, for example
by setting all non-observed entries to zero. For a numerical example, return to the matrix $M$ in \eqref{eq:M}. Note that we observe only three instances hence the classification task is easy but doesn't lead to reconstruction of the missing entries.

\section{Problem Setup and Main Result}\label{sec:prob}

We begin by presenting the general setting: A vector with missing entries can be modeled as a tuple $\xx\times \oo$, where $\xx\in \mathbb{R}^d$ and $\oo\in 2^d$ is a subset of indices. The vector $\xx$ represents the \emph{full data} and the set $\oo$ represents the \emph{observed attributes}. Given such a tuple, let us denote by $\xx_{\oo}$ a vector in $(\mathbb{R}\cup\{*\})^d$ such that

\[(\xx_{\oo})_i = \begin{cases}
\xx_i & i \in \oo \\
* & \mathrm{else} \end{cases}\]

The task of learning a linear classifier with missing data is to return a target function over $\xx_\oo$ that competes with best linear classifier over $\xx$. Specifically, a sequence of triplets $\{(\sampx{i}.\sampo{i}.y_i)\}_{i=1}^m$ is drawn iid according to some distribution $D$. An algorithm is provided with the sample $S=\{(\sampox{i},y_i)\}_{i=1}^m$ and should return a target function $f_S$ over missing data such that w.h.p:
\begin{equation}\label{eq:main} \mathop\mathbb{E}\left[\ell (f_S(\xx_\oo),y))\right] \le \min_{ \ww \in B_d(1) } \mathop\mathbb{E}\left[ \ell ( \ww\cdot \xx,y))\right] +\epsilon,
 \end{equation}
where $\ell$ is the loss function and $B_d(r)$ denotes the Euclidean ball in dimension $d$ of radius $\sqrt{r}$.
For brevity, we will say that a target function $f_S$ is \epg if \eqref{eq:main} holds.

Without any assumptions on the distribution $D$, the task is ill-posed. One can construct examples where the learner over missing data doesn't have enough information to compete with the best linear classifier. Such is the case when, e.g., $y_i$ is  some attribute that is constantly concealed and independent of all other features.  Therefore, certain assumptions on the distribution must be made.

One reasonable assumption is to assume that the marginal distribution $D$ over $\xx$ is supported on a small dimensional linear subspace $E$ and that for every set of observations, we can linearly reconstruct the vector $\xx$ from the vector $P_{\oo}\xx$, where $P_\oo: \mathbb{R}^d \to \mathbb{R}^{|\oo|}$ is the projection on the observed attributes. In other words, we demand that the mapping ${P_{\oo}}_{|E}:E \to P_{\oo}E$, which is the restriction of $P_{\oo}$ to $E$, is full-rank. As the learner doesn't have access to the subspace $E$, the learning task is still far from trivial.

We give a precise definition of the last assumption in \asref{as:lowrank}. Though our results hold under the low rank assumption the convergence rates we give depend on a certain regularity parameter. Roughly,  we parametrize the "distance" of ${P_{\oo}}_{|E}$  from singularity, and our results will quantitively depend on this distance. Again, we defer all rigorous definitions to \secref{sec:assumptions}.

Our first result is a an upper bound on the sample complexity of the problem. We then proceed to a more general statement that entails an efficient kernel-based algorithm.

\subsection{Main Result}\label{sec:main}
\begin{theorem}[Main Result]\label{thm:main}
Assume that $\ell$ is a $L$-Lipschitz convex loss function  Let $D$ be a $\lambda$-regular distribution (see \defref{def:regularity})
Let $\gamma(\epsilon) \ge \frac{\log 2L /(\lambda\epsilon)}{\lambda}$ and \[\Gamma(\epsilon)= \frac{d^{\gamma(\epsilon)+1}-d}{d-1}.\]
There exists an algorithm (independent of $D$) that receives a sample $S= \{(\sampox{i},y_i)\}_{i=1}^m$ of size $m\in \Omega \left(\frac{ L^2 \Gamma(\epsilon)^2\log 1/\delta}{\epsilon^2}\right)$ and returns a target function $f_S$ that is \epg with probability at least $(1-\delta)$. The algorithm runs in time $\mathrm{poly}(|S|)$.
\end{theorem}

\thmref{thm:main} gives an upper bound on the computational and sample complexity of learning a linear classifier with missing data under the low rank assumption. As the sample complexity is quasipolynomial, this has limited practical value in many situations. However, as the next theorem states, $f_S$ can actually be computed by applying a kernel trick. Thus, under further \emph{large margin} assumptions we can significantly improve performance.

\begin{theorem}\label{thm:main2} 
For every $\gamma \ge 0$, there exists an embedding over missing data \[\phi_\gamma: \xx_\oo \to \mathbb{R}^\Gamma\]
such that $\Gamma= \sum_{k=1}^\gamma d^{k}=\frac{d^{\gamma+1}-d}{d-1}$, and the scalar product between two samples $\phi_\gamma(\sampox{1})$ and $\phi_\gamma(\sampox{2})$ can be efficiently computed, specifically it is given by the formula:

\[
k_\gamma(\sampox{1},\sampox{2}):=
 \frac{|\sampo{1}\cap\sampo{2}|^{\gamma}-1}{|\sampo{1}\cap\sampo{2}|-1}  \times \sum_{i\in \sampo{1}\cap\sampo{2}} \sampx{1}_i\cdot \sampx{2}_i.
\]
In addition, let $\ell$ be an $L$-Lipschitz loss function and $S=\{\sampox{i}\}_{i=1}^m$ a sample drawn iid according to a distribution $D$. We make the assumption that $\|P_{\oo}\xx\|\le 1$ a.s. The followings hold:
\begin{enumerate}
\item\label{thm2:efficient} At each iteration of \algref{alg:GDMD} we can efficiently compute $\vv_t^\top \phi_\gamma(\sampox{t})$ for any new example $\sampox{t}$. Specifically it is given by the formula \[\vv_t^\top\phi_{\gamma}(\sampox{t}):=\sum_{i=1}^t {\alpha^{(t)}_i}k(\sampox{i},\sampox{t}).\] Hence \algref{alg:GDMD} runs in $\mathrm{poly}(T)$ time and sequentially produces target functions $f_t(\xx_\oo) = \vv_t^\top \phi_\gamma(\xx_\oo)$ that can be computed at test time in $\mathrm{poly}(T)$ time.
\item\label{thm2:optimal} Run \algref{alg:GDMD} with $\eta_t = \frac{C}{t}$, $\rho=1$ and $T$. Let $\bar{\vv}=\frac{1}{T} \sum_{t=1}^m \vv_t$, then with probability $(1-\delta)$:
\begin{equation}\label{eq:optimization} \frac{1}{2} \|\bar{\vv}\|^2 +  \frac{C}{m} \sum_{i=1}^m   \ell ( \bar{\vv}^\top\phi_\gamma(\sampox{i}) ,y_i) \le   \min \frac{1}{2}\|\vv\|^2 + \frac{C}{m}\sum_{i=1}^m \left[\ell( \vv^\top \phi_\gamma(\sampox{i}),y_i)\right] + \tilde{O}\left(\frac{ (C L)^2\Gamma\ln 1/\delta}{T} \right).\end{equation}
\item\label{thm2:competetive} For any $\epsilon>0$, if $D$ is a $\lambda$-regular distribution and  $\gamma \ge \frac{\log 2{L}/(\lambda\epsilon)}{\lambda}$ then for some $\vv^* \in B_{\Gamma}(\Gamma)$
\[ \mathbb{E}\left[ \ell ( \vv^*\cdot \phi_\gamma(\xx_\oo ,y )\right] \le \min_{\ww \in B_d(1) }\mathbb{E}\left[\ell( \ww\cdot \phi_\gamma(\xx_\oo),y)\right] + \epsilon.\]
\end{enumerate}
\end{theorem}
To summarize, \thmref{thm:main2} states that we can embed the sample points with missing attributes in a high dimensional, finite, Hilbert space of dimension $\Gamma$, such that:
\begin{itemize}
\item The scalar product between embeded points can be computed efficiently. Hence, due to the conventional representer argument, the task of empirical risk minimization is tractable.
\item Following the conventional analysis of kernel methods: Under large margin assumptions in the ambient space, we can compute a predictor with scalable sample complexity and computational efficiency.

\item Finally, the best linear predictor over embedded sample points in a $\sqrt{\Gamma}$--ball is comparable to the best linear predictor over fully observed data.
\end{itemize}
 Taken together, we can learn a predictor with sample complexity $\Omega(\Gamma^2(\epsilon)/\epsilon^2\log\frac{1}{\delta})$ and \thmref{thm:main} holds.
 
   For completeness we present the method together with an efficient algorithm that optimizes the RHS of \eqref{eq:optimization} via an SGD method.  The optimization analysis is derived in a straightforward manner from the work of \cite{shalev2011pegasos}. Other optimization algorithms exist in the literature, and we chose this optimization method as it allows us to also derive regret bounds which are formally stronger (see \secref{sec:regret}). We stress that the main novelty of this paper is not in any specific optimization algorithm, but the introduction of a new kernel and our guarantees rely solely on it.

Finally, note that $\phi_1$ induces the same scalar product as a $0$-imputation. In that respect, by considering different $\gamma=1,2, \ldots$ and using a holdout set we can guarantee that our method will outperform the $0$-imputation method. By normalizing or adding a bias term we can in fact compete with mean-imputation or any other first order imputation.

\subsection{Regret minimization for joint subspace learning and classification}\label{sec:regret}

A significant technical contribution of this manuscript is the agnostic learning of a subspace coupled with a linear classifier.  A subspace is represented by a projection matrix $Q \in \reals^{d \times d}$, which satisfies $Q^2 = Q$. Denote the following class of target functions \[\mathcal{F}_0= \{f_{\ww,Q}: \ww \in B_d,  ~ Q\in M_{d\times d}, ~ Q^2=Q\}\] where
$ f_{\ww,Q} (\xx_\oo) $ is the linear predictor defined by $\ww$ over subspace defined by the matrix $Q$, as formally defined in definition \ref{def:F0}.

Given the aforementioned efficient kernel mapping $\phi_\gamma$, we consider the following kernel-gradient-based online algorithm for classification called KARMA (Kernelized Algorithm for Risk-minimization with Missing Attributes).

 \begin{algorithm}[h!]
 	\caption{KARMA: Kernelized Algorithm for Risk-minimization with Missing Attributes \label{alg:GDMD} }
 	\begin{algorithmic}[1]
 		\STATE Input: parameters $\gamma > 1, ~ \{\eta_t > 0 \} , 0 < \rho  < 1, B > 0 $
 		\FOR {$t=1$ to $T$}
		\STATE Observe example $(\xx^t_{\oo^t},y_t) $, suffer loss $\ell( \vv_t^\top \phi_\gamma(\xx^t_{\oo^t}) , y_t)$
 		\STATE Update 
 		$$\alpha_i^{(t)} = \begin{cases} (1 - \eta_t \rho ) \cdot \alpha_i^{(t-1)} & i < t \\ -\eta_t\ell'(\vv_t^\top \phi_\gamma(\xx^t_{\oo^t})) & i=t \\ 0 & \mathrm{else} \end{cases} $$
		$$ \vv_{t+1} = \sum_{i=1}^t \alpha_i^{(t)}\phi_\gamma(\xx^i_{\oo^i}) $$
		\ENDFOR
 	\end{algorithmic}
 \end{algorithm}
 
Our main result for the fully adversarial online setting is given next, and proved in the Appendix. Notice that the subspace $E^*$ and associated projection matrix $Q^*$ are chosen by an adversary and unknown to the algorithm. 

 \begin{theorem}\label{thm:regret}
 For any $\gamma > 1, \lambda>0, X > 0, \rho > 0, B > 0$,  $L$-Lipschitz  convex loss function $\ell$,  and $\lambda$-regular sequence  $\{ (\xx^t,\oo^t,y_t) \}$ w.r.t subspace $E^*$ and associated projection matrix $Q^*$ such that $\|\xx^t\|_\infty<X$, Run Algorithm \ref{alg:GDMD}  with $\{ \eta_t = \frac{ 1}{ \rho t} \}$, sequentially outputs $\{ \vv_t \in R^t \}$ such that
$$ \sum_t \ell( \vv_t ^\top \phi_\gamma(\xx^t_{\oo_t}) , y_t ) -   \min_{ \|\ww\|\leq 1 } \sum_t \ell(f_{\ww,Q^*} (\xx^t_{\oo_t}) ,y_t) \leq \frac{2 L^2 X^2 \Gamma} {\rho}(1 + \log T) + \frac{\rho}{2} T \cdot B   + \frac{e^{-\lambda \gamma} } {\lambda} L T$$ 

In particular, taking $\rho = \frac{LX \sqrt{\Gamma}}{\sqrt{BT}}$,  $\gamma  = \frac{1}{\lambda} \log {T} $ we obtain 

$$ \sum_t \ell( \vv_t ^\top \phi_\gamma(\xx^t_{\oo_t}) , y_t ) -   \min_{ \|\ww\|\leq 1} \sum_t \ell(f_{\ww,Q^*} (\xx^t_{\oo_t}) ,y_t) = O(    X L \sqrt{\Gamma B T}  ) $$ 

 \end{theorem}

\section{Preliminaries and Notations}\label{sec:prem}
\subsection{Notations}\label{sec:notations}
As discussed, we consider a model where a distribution $D$ is fixed over $\mathbb{R}^d \times \mathcal{O}\times \mathcal{Y}$, where $\mathcal{O}=2^d$ consists of all subsets of $\{1,\ldots,d\}$.
We will generally denote elements of $\mathbb{R}^d$ by $\xx,\ww,\vv,\uu$ and elements of $\mathcal{O}$ by $\oo$. We denote by $B_d$ the unit ball of $\mathbb{R}^d$, and by $B_d(r)$ the ball of radius $\sqrt{r}$.

Given a subset $\oo$ we denote by $P_{\oo}:\mathbb{R}^d \to \mathbb{R}^{|\oo|}$ the projection onto the indices in $\oo$, i.e., if $i_1\le i_2\le\cdots \le i_k$ are the elements of $\oo$ in increasing order then $(P_{\oo} \xx)_j = \xx_{i_j}$.
 Given a matrix $A$ and a set of indices $\oo$, we let \[A_{\oo,\oo}= P_{\oo} A P_{\oo}^\top.\]

\subsection{Model Assumptions}\label{sec:assumptions}

\begin{definition}[$\lambda$-regularity]\label{def:regularity}
 We say that $D$ is $\lambda$-regular with associated subpsace $E$  if the following happens with probability $1$ (w.r.t the joint random variables $(\xx,\oo)$):
\begin{enumerate}

\item\label{as:compactness} $\|P_{\oo} \xx\|\le 1$.
\item\label{as:support} $\xx \in E$.
\item\label{as:reconstructible} $\ker(P_\oo P_E) = \ker(P_E)$ 
\item\label{as:nonsingular} If $\lambda_{\oo}>0$ is a strictly positive singular value of the matrix $P_{\oo}P_E$ then $\lambda_\oo \ge \lambda$.
\end{enumerate}
\end{definition}
\begin{assumption}[Low Rank Assumption]\label{as:lowrank}
We say that $D$ satisfies the low rank assumption with asscoicated subspace $E$ if it is $\lambda$-regular with associated subspace $E$ for some $\lambda>0$.
\end{assumption}

 Note that in our setting we assume that $\|P_{\oo}\xx\|\le 1$ a.s. If $\|\xx\|\le 1$ then $\|P_{\oo}\xx\|\le 1$ hence our assumption is weaker then assuming $\xx$ is contained in a fixed sized ball. Further, the assumption can be verified on a sample set with missing attributes.

 Note also that we've normalized both $\ww$ and $\xx_{\oo}$. To achieve guarantees that scale with $\|\ww\|$, note that we can replace the loss function $\ell(\ww\cdot \xx,y)$ with $\ell(\rho \cdot\ww\cdot \xx,y)$ for any constant $\rho$. This will replace $L$--Lipschitness with $\rho\cdot L$--Lipschitzness in all results.

\section{Learning under low rank assumption and $\lambda$-regularity.}
\begin{definition}[The class $\mathcal{F}_0$] \label{def:F0}
We define the following class of target functions \[\mathcal{F}_0= \{f_{\ww,Q}: \ww \in B_d(1),  ~ Q\in M_{d\times d}, ~ Q^2=Q\}\] where
\[f_{\ww,Q} (\xx_\oo) = (P_{\oo} \ww) \cdot Q_{\oo,\oo}^\dagger \cdot (P_{\oo}\xx).\]
(Here $M^\dagger$ denotes the pseudo inverse of $M$.)
\end{definition}
The following Lemma states that, under the low rank assumption, the problem of linear learning with missing data is reduced to the problem of learning the class $\mathcal{F}_0$, in the sense that the hypothesis class $\mathcal{F}_0$ is not less-expressive.
\begin{lemma}\label{lem:lintof0}
Let $D$ be a distribution that satisfies the low rank assumption. For every $\ww^* \in \mathbb{R}^d$ there is $f^*_{\ww,Q} \in \mathcal{F}_0$ such that a.s:
\[ f^*_{\ww,Q}(\xx_\oo) = \ww^*\cdot \xx.\]
In particular $Q=P_E$ and $\ww = P^\top_E \ww^*$ , where $P_E$ is the projection matrix on the subspace $E$.
\end{lemma}

\subsection{Approximating $\mathcal{F}_0$ under regularity}
We next define a surrogate class of target functions that approximates $\mathcal{F}_0$
\begin{definition}[The classes $\mathcal{F}^\gamma$]
For every $\gamma$ we define the following class
\[ \mathcal{F}^\gamma=\{ f^\gamma_{\ww,Q} : \ww\in B_d(1), ~ Q\in \mathbb{R}^{d\times d}, ~Q^2=Q\}\]
where,
\[f^\gamma_{\ww,Q}(\xx_\oo) = (P_{\oo}\ww) \cdot \sum_{j=0}^{t-1} \left(Q_{\oo,\oo}\right)^j\cdot  (P_{\oo} \xx)\]
\end{definition}
\begin{lemma}\label{lem:approximation}
Let $(\xx,\oo)$ be a sample drawn according to a $\lambda$-regular distribution $D$ with associated subspace $E$. Let $Q=P_E$ and $\|\ww\|\le 1$ then a.s:

\[ \|f^\gamma_{\ww, I-Q}(\xx_\oo) - f_{\ww,Q }(\xx_\oo)\| \le \frac{(1-\lambda)^{\gamma}}{\lambda}.\]
\end{lemma}
\begin{corollary}\label{cor:learnft} Let $\ell$ be a $L$-Lipschitz function. Under $\lambda$-regularity, for every $\gamma \ge \frac{\log L/\lambda \epsilon}{\lambda}$
the class $\mathcal{F}^\gamma$ contains an \epg target function.
\end{corollary}

\subsection{Improper learning of $\mathcal{F}^\gamma$ and a kernel trick}

Let $\mathbb{G}$ be the set of all finite, non empty, sequences of length at most $\gamma$ over $d$. For each $\ss \in \mathbb{G}$ denote $|\ss|$-- the length of the sequence and $\ss_{\mathrm{end}}$ the last element of the sequence. Given a set of observations $\oo$ we write $\ss\subseteq \oo$ if all elements of the sequence $\ss$ belong to $\oo$. We let 
$$\Gamma=\sum_{j=1}^{\gamma} d^j=|\mathbb{G}| = \frac{d^{\gamma+1} - d}{d - 1} $$ 
and we index the coordinates of $\mathbb{R}^\Gamma$ by the elements of $\mathbb{G}$:

\begin{definition}\label{def:phi}
We let $\phi_\gamma: (\mathbb{R}^d \times \mathcal{O}) \to \mathbb{R}^{\Gamma}$ be the embedding:

\[\left(\phi_\gamma(\xx_\oo)\right)_{\ss} =
\begin{cases}
\xx_{\ss_{\mathrm{end}}} & \ss \subseteq \oo \\
0 & \textrm{else}
\end{cases}\]
\end{definition}

\begin{lemma}\label{lem:ftunfold}

For every $Q$ and $\ww$ we have:

\[f^\gamma_{\ww,Q}(\xx_\oo)= \sum_{\ss_1 \in \oo} \ww_{\ss_1}\xx_{\ss_1} +\sum_{\{\ss: \ss\subseteq \oo, ~2\le |\ss|\le t\}} \ww_{\ss_1}\cdot Q_{\ss_1,\ss_2} \cdot Q_{\ss_2,\ss_3}\cdots Q_{\ss_{|\ss|-1}, \ss_{\mathrm{end}}} \cdot \xx_{\ss_{\mathrm{end}}}\]
\end{lemma}

\begin{corollary}\label{cor:improper}
For every $f^\gamma_{\ww,Q} \in \mathcal{F}^\gamma$ there is $\vv \in   B_{\Gamma}(\Gamma) $, such that:

\[f^\gamma_{\ww,Q}(\xx_\oo) = \vv\cdot \phi_\gamma(\xx_\oo).\] As a corllary, for every loss function $\ell$ and distribution $D$ we have that:

\[ \min_{\vv\in B_{\Gamma}(\Gamma)} \mathop\mathbb{E}\left[\ell ( \vv\cdot \phi(\xx_\oo),y)\right] \le \min_{f_{\ww,Q}^\gamma \in \mathcal{F}^\gamma} \mathop\mathbb{E} \left[\ell (f^\gamma_{\ww,Q}(\xx_\oo),y)\right]\]
\end{corollary}

Due to \coref{cor:improper}, learning $\mathcal{F}^\gamma$ can be improperly done via learning a linear classifier over the embedded sample set $\{\phi_\gamma(\xx_\oo)\}_{i=1}^m$. While the ambient space $\mathbb{R}^\Gamma$ may be very large, the computational complexity of the next optimization scheme is actually dependent on the scalar product between the embedded samples. For that we give the following result that shows that the scalar product can be computed efficiently:

\begin{theorem}\label{thm:ktrick}
\[\phi_\gamma(\xx^{(1)}_{\oo_1}) \cdot \phi_\gamma(\xx^{(2)}_{\oo_2}) =\frac{|\oo_1\cap \oo_2|^\gamma-1}{|\oo_1\cap \oo_2|-1} \sum_{k \in \oo_1\cap \oo_2} \xx^{(1)}_k \xx^{(2)}_k.\]
(We use the convention that $\frac{1^\gamma-1}{1-1} =\lim_{x\to 1} \frac{x^\gamma-1}{x-1}= \gamma$)
\end{theorem}

\ifICML
\section{Experiments}\label{sec:exp}

\input{experiments}

\fi

\section{Discussion and future work}

We have described the first  theoretically-sound method to cope with low rank missing data, giving rise to a classification algorithm that attains competitive error to that of the optimal linear classifier that has access to all data. Our non-proper agnostic framework for learning a hidden low-rank subspace  comes with provable guarantees, whereas  heuristics based on separate data reconstruction and classification are shown to  fail for certain scenarios. 

Our technique is directly applicable to classification with low rank missing data and  polynomial kernels via kernel (polynomial) composition.  General kernels can be handled by polynomial approximation, but it is interesting to think about a more direct approach. 

It is possible to derive all our results for a less stringent condition than $\lambda$-regularity: instead of bounding the smallest eigenvalue of the hidden subspace, it is possible to bound only the ratio of largest-to-smallest eigenvalue. This results in better bounds in a straightforward plug-and-play into our analysis, but was ommitted for simplicity.

\bibliography{olc}
\bibliographystyle{icml2014}

\ifICML
\newpage
\onecolumn
\fi

\appendix

\section{Proofs of theorems and lemmas from main text}

\subsection{Technical Claims}
\begin{claim}\label{claim:technical1} 
Let $Q\in M_{d\times d}$ be a square projection matrix and $P\in M_{k\times d}$ a matrix. Recall that:
\[ \im (A)= \{\vv:\exists \uu ~ A\uu=\vv\}, \quad \mathrm{and} \quad \ker(A)=\{\vv: A\vv=0\}.\] And that $\rank(A)$ is the size of the largest collection of linearly independent columns of A.

The following statements are equivalent:

\begin{enumerate}
\item\label{tech:kernel} $\ker(PQ)=\ker(Q)$.
\item\label{tech:rank} $\rank(PQ) = \rank(QP^\top)=\rank(P Q P^\top)=\rank(Q)$.
\item\label{tech:image} $\im( QP^\top) = \im(Q)$.
\end{enumerate}

\end{claim}
\begin{proof}~

\begin{itemize}
\item[\ref{tech:kernel} $\Rightarrow$ \ref{tech:rank}]
Clearly $\rank(P Q)\le \rank (Q)$. If $\rank(P Q) < \rank(Q)$ we must have some collection of linearly independent columns of $Q$ that are linearly dependent in $PQ$ this implies that there is $\vv$ such that $P Q \vv=0$ but $Q\vv\ne 0$. Hence $\ker(PQ) \neq \ker(Q)$ and thus a contradiction, we conclude that $\rank(P Q)=\rank(Q)$.

That $\rank(P Q) = \rank (QP^\top)=\rank(P Q P^\top)$ follows from the fact that $\rank(A)=\rank(A^\top)=\rank(AA^\top)$ and using the fact that $Q^2=Q$ since $Q$ is a projection matrix.

\item[\ref{tech:rank}$\Rightarrow$ \ref{tech:image}]
We have that $\mathrm{Im}(QP^\top) \subseteq \mathrm{Im}(Q)$. The two subspaces, $\im(QP^{\top})$ and $\im(Q)$, are in fact the linear span of the columns of $QP^\top$ and $Q$ respectively.

 Since $\rank(QP^\top)=\rank(Q)$ we conclude that the dimension of the two subspaces is equal. It follows that $\mathrm{Im}(QP^\top) = \mathrm{Im} (Q)$.

\item[\ref{tech:image} $\Rightarrow$ \ref{tech:kernel}] 
Since $\im(QP^\top)=\im(Q)$ we also have $\rank(QP^\top)=\rank(Q)$ and as a corollary $\rank(PQ)=\rank(Q)$.

Now by the rank-nullity Theorem, for every $A\in M_{k\times d}$, $\dim(\ker(A)) = d-\rank(A)$.

  Hence $\dim(\ker(PQ))=\dim(\ker(Q))$. Since $\ker(PQ)\subseteq \ker(Q)$ we must have $.\ker(PQ)=\ker(Q)$.
\end{itemize}\end{proof}

\begin{claim}\label{claim:technical2}
Let $\oo \in 2^d$ be drawn according to a distribution $D$ that satisfies the low rank assumption. If $Q=P_E$ then: 
\[\im(Q_{\oo,\oo}) = \im (P_{\oo} Q)\]
\end{claim}
\begin{proof}
$\ker(P_\oo Q) = \ker (Q)$ holds by assumption (assumption \ref{as:reconstructible} in \defref{def:regularity}). $\im(Q) = \im (QP_{\oo}^\top)$ then follows from item \ref{tech:image}. In particular $\im(P_{\oo}Q)= \im(P_{\oo}Q{P_{\oo}}^\top)=\im(Q_{\oo,\oo})$.
\end{proof}

\subsection{proof of \lemref{lem:lintof0}}
By definition, if $P_\oo \xx \in \im(Q_{\oo,\oo})$ then $Q_{\oo,\oo}\left(Q_{\oo,\oo}\right)^\dagger P_{\oo} \xx = P_\oo \xx$.  We claim that due to the low rank assumption, $P_\oo \xx \in \im(Q_{\oo,\oo})$. 

Indeed, recall that $Q=P_E$ and $\xx \in E$ hence $Q\xx=\xx$ and $P_{\oo}\xx \in \im(P_{\oo}Q)$. By \claimref{claim:technical2} we have $\im(Q_{\oo,\oo})=\im (P_{\oo} Q)$, hence $P_{\oo}\xx\in (\im Q_{\oo,\oo})$.

Next, we have that 

\[P_{\oo} Q P_{\oo}^\top \left(Q_{\oo,\oo}\right)^\dagger P_{\oo}\xx = Q_{\oo,\oo}\left(Q_{\oo,\oo}\right)^\dagger P_{\oo}\xx=P_{\oo}\xx\]
Alternatively

\begin{equation}P_{\oo}(Q P_{\oo}^\top Q^\dagger_{\oo,\oo}P_{\oo}\xx-\xx)=0. \end{equation}
Again, since $Q\xx= \xx$ we have that:

\begin{equation}\label{eq:pnull} P_{\oo}Q ( P_{\oo}^\top Q^\dagger_{\oo,\oo}P_{\oo}\xx-\xx)=0. \end{equation}
The low rank assumption implies that $P_{\oo}Q \vv=0$ if and only if $Q\vv=0$. Apply this to $\vv = P_{\oo}^\top Q^\dagger_{\oo}P_{\oo}\xx-\xx$ and get:

\[Q P_{\oo}^\top Q^\dagger_{\oo,\oo}P_{\oo}\xx=Q\xx=\xx.\]

Finally we have that

\[f_{\ww,Q}(\xx_\oo) = (P_{\oo} Q^\top\ww^*)\cdot  Q_{\oo,\oo}^\dagger P_{\oo}\xx =  \ww^* \cdot Q  P_{\oo}^\top Q_{\oo,\oo}^\dagger P_{\oo}\xx =\ww^*\cdot \xx.\] 


\subsection{proof of \lemref{lem:approximation}}
Let $I$ denote the identity matrix in $M_{d\times d}$.
First note that $(I_{\oo,\oo}-Q_{\oo,\oo}) = (I-Q)_{\oo,\oo}$ and that $I_{\oo,\oo}$ is the identity matrix in $\mathbb{R}^{|\oo|\times |\oo|}$.

Let $\vv_1,\ldots,\vv_k$ be the normalized and orthogonal eigen-vectors of $Q_{\oo,\oo}$ with strictly positive eigenvalues $\lambda_1\ge \ldots,\lambda_k$. By $\lambda$-regularity we have that $\lambda_k \ge \lambda$ and since the spectral norm of $Q_{\oo,\oo}$ is smaller than the spectral norm of $Q$ we have that $\lambda_1\le 1$.

Note that for every $\vv_j$ we have $Q_{\oo,\oo}^\dagger \vv_j = \frac{1}{\lambda_j}\vv_j$. Next, recall that $Q=P_E$ and $\xx \in E$ hence $Q\xx=\xx$ and $P_\oo \xx \in \im(P_\oo Q)$. By \claimref{claim:technical2} we have $P_{\oo}\xx \in \im(Q_{\oo,\oo})$. Since $\im(Q_{\oo,\oo})=\mathrm{span}(\vv_1,\ldots, \vv_k)$, we may write $P_{\oo}\xx = \sum \alpha_i \vv_i$. Since $\|P_{\oo}\xx\|\le 1$ and $\{\vv_1,\ldots,\vv_k\}$ is an orthonormal system we have $\sum \alpha_i^2 \le 1$.

Hence

\[ \|\left(\sum_{j=0}^{\gamma-1} (I_{\oo,\oo}-Q_{\oo,\oo})^j - Q^{\dagger}_{\oo,\oo}\right)P_{\oo}\xx\|= \|\sum_{i} \alpha_i \left( \sum_{j=0}^{\gamma-1} (1-\lambda)^j_i - \frac{1}{\lambda_i}\right) \vv_i\| \le \max_{i} \left| \sum_{j=0}^{\gamma-1}(1- \lambda_i)^j -\frac{1}{\lambda_i}\right|
\le\]\[ \max_{i} \left| \frac{1-(1-\lambda_i)^{\gamma}}{\lambda_i}  -\frac{1}{\lambda_i}\right|\le \frac{(1-\lambda)^{\gamma}}{\lambda}.\]
Finally since $\|P_{\oo}\ww\|\le 1$ we get that
\[\|f^\gamma_{\ww,I-Q}(\xx_\oo) - f_{\ww,Q}(\xx_\oo)\|\le \frac{(1-\lambda)^{\gamma}}{\lambda}\]

\subsection{Proof of \lemref{lem:ftunfold}}
Let $\oo_1\le \oo_2,\le \ldots\le \oo_{|\oo|}$ be the elements of $\oo$ ordered in increasing order. First by definition we have that:
\begin{equation}\label{eq:g} f^\gamma_{\ww,Q} (\xx_\oo) = \sum_{j=0}^{\gamma-1} \sum_{n,k=1}^{|\oo|} \ww_{\oo_n}((Q_{\oo,\oo})^j)_{n,k} \xx_{\oo_k}=
\sum_{i\in \oo} \ww_i \xx_i +  \sum_{j=1}^{\gamma-1} \sum_{n,k=1}^{|\oo|} \ww_{\oo_n}((Q_{\oo,\oo})^j)_{n,k} \xx_{\oo_k}
\end{equation}

We also have by definition that for $j\ge 1$: \[((Q_{\oo,\oo})^j)_{n,k} = \sum_{s=1}^{|\oo|} ((Q_{\oo,\oo})^{j-1})_{n,s}  ((Q_{\oo,\oo}))_{s,k}= \sum_{s=1}^{|\oo|} ((Q_{\oo,\oo})^{j-1})_{n,s} Q_{\oo_s,\oo_k}\] 
By induction we can show that: 
\[((Q_{\oo,\oo})^j)_{n,k} = \sum_{\ss_1\in \oo} Q_{\oo_n,\ss_1}\left( \sum_{\ss_2\in \oo} Q_{\ss_1,\ss_2}\left(\sum \cdots \left(\sum_{\ss_{j-1}\in \oo} Q_{\ss_{j-2},\ss_{j-1}}Q_{\ss_{j-1},\oo_k}\right)\cdots\right)\right).\]
Reordering the elements we get for $j\ge 1$:
\begin{equation}\label{eq:Q}((Q_{\oo,\oo})^j)_{n,k}  =\sum_{\{\ss: |\ss|=j+1,\ss_1=\oo_n,\ss_{j+1}=\oo_k\}} Q_{\ss_1,\ss_2}\cdot Q_{\ss_2,\ss_3}\cdots Q_{\ss_{j},\ss_{j+1}}\end{equation}
The result now follows from \eqref{eq:g} and \eqref{eq:Q} by a change of indexes.

\subsection{Proof of \coref{cor:improper}}
Choose 
\[\vv_{\ss} =\begin{cases}

 \ww_{\ss_1} & |\ss|=1 \\
 \ww_{\ss_1}\cdot Q_{\ss_1,\ss_2} \cdot Q_{\ss_2,\ss_3}\cdots Q_{\ss_{|\ss|-1}, \ss_{\mathrm{end}}} &  |\ss|>1
\end{cases}\]
It is clear from \lemref{lem:ftunfold} that $f^\gamma_{\ww,Q}(\xx_\oo) = \vv\cdot \phi_\gamma(\xx_\oo)$. We only need to show that $\|\vv\|\le \sqrt{\Gamma} \|\ww\|$.

 Note that since $Q^2 = Q$ we have $\max(|Q_{i,j}|)<1$. Hence $|\vv_\ss| \le |\ww_{\ss_1}|$ and:
\[\|\vv\|^2 = \sum_{\ss \in \mathbb{G}} \vv^2_{\ss}\le \sum_{\ss\in \mathbb{G}} \ww^2_{\ss_1}\le \Gamma \|\ww\|^2\]

\subsection{Proof of \thmref{thm:ktrick}}
By definition of $\phi_\gamma$ we have:
\[\phi_\gamma(\xx^{(1)}_{\oo_1}) \cdot \phi_\gamma(\xx^{(2)}_{\oo_2}) = \sum_{\ss \subseteq \oo_1\cap \oo_2} \xx^{(1)}_{\ss_{\mathrm{end}}} \cdot \xx^{(2)}_{\ss_{\mathrm{end}}}
= \sum_{l=1}^\gamma \sum_{k\in \oo_1\cap \oo_2}~  \sum_{\ss \subseteq \oo_1\cap \oo_2, \ss_{\mathrm{end}}=k,|\ss|=l} \xx^{(1)}_{k} \cdot \xx^{(2)}_{k}\]\[
=  \sum_{l=1}^1~ \sum_{|\ss|=l-1, \ss\subset \oo_1\cap \oo_2}~\sum_{k\in \oo_1\cap\oo_2}\xx^{(1)}_{k} \cdot \xx^{(2)}_{k} \]\[
= \sum_{l=1}^1 |\ss: |\{|\ss|=l-1, ~  \ss\subset \oo_1\cap \oo_2\}| \sum_{k\in \oo_1\cap\oo_2}\xx^{(1)}_{k} \cdot \xx^{(2)}_{k}= \sum_{l=1}^\gamma |\oo_1\cap \oo_2|^{l-1}  \cdot \sum_{k\in \oo_1\cap\oo_2}\xx^{(1)}_{k} \cdot \xx^{(2)}_{k}=\]\[\frac{1- |\oo_1\cap \oo_2|^\gamma}{1-|\oo_1\cap \oo_2|} \sum_{k \in \oo_1\cap \oo_2} \xx^{(1)}_k \xx^{(2)}_k\]

\subsection{Proof of \thmref{thm:main2}} 

We take $\phi_\gamma$ as in \defref{def:phi}. That $\phi_\gamma(\sampox{1})\cdot\phi_\gamma(\sampox{2})=\frac{1-|\sampo{1}\cap\sampo{2}|^\gamma}{1-|\sampo{1}\cap\sampo{2}|}\sum_{i\in \sampo{1}\cap\sampo{2}} \sampx{1}_i\cdot\sampx{2}_i$ is shown in \thmref{thm:ktrick}.

The analysis of sub-gradient descent methods to optimize problems of this form i.e: \[\|\ww\|^2 + \frac{C}{m}\sum_{i=1}^m (\ell(\ww^\top \phi(\xx_i),y_i).\] was studied in \cite{shalev2011pegasos} and the detailed analysis can be found there (with generalization to mercer kernels and general losses). We mention that since $\ell$ is $L$-Lipschitz and $\|\phi_{\gamma}(\xx_\oo)\|\le \sqrt{\Gamma}$ a bound $R$ on the gradient of $\nabla \ell(\vv^\top \phi_{\gamma}(\xx_\oo),y)=\ell'(\vv^\top\phi_\gamma(\xx_\oo),y)\phi_{\gamma}(\xx_\oo)$ is given by $L\sqrt{\Gamma}$.

This establishes items \ref{thm2:efficient} and \ref{thm2:optimal}.

Next we let $\ell$ be an $L$-Lipschitz loss function and $D$ a $\lambda$-regular distribution and we assume that $\gamma \ge \frac{\log 2{L}/(\lambda\epsilon)}{\lambda}$.

Due to \coref{cor:improper}, for some $\vv^* \in B_\Gamma(\Gamma) $
\[ \mathbb{E} \left[\ell(\vv^* \cdot\phi_\gamma(\xx_\oo),y)\right] \le \min_{f^\gamma_{\ww,I-Q}\in \mathcal{F}^\gamma}\mathbb{E}\left[\ell(f^\gamma_{\ww,I-Q}(\xx_\oo),y)\right]\]
Applying \lemref{lem:approximation} and ${L}$-Lipschitness, for every $f^*_{\ww,Q} \in \mathcal{F}_0$ we have:
\[\mathbb{E}\left[\ell(\vv^* \cdot\phi_\gamma(\xx_\oo),y)\right] \le \mathbb{E}\left[ \ell(f^*_{\ww,Q} (\xx_\oo,y))\right]+ {L}\frac{(1-\lambda)^\gamma}{\lambda}.\]
The result follows from choice of $\gamma$ and the inequality $-\lambda>\log (1-\lambda)$:
\[\frac{(1-\lambda)^\frac{\log 2{L}/(\lambda \epsilon)}{\lambda}}{\lambda} \le \frac{(1-\lambda)^\frac{\log (\lambda\epsilon)/(2{L})}{\log (1-\lambda)}}{\lambda}=\frac{\epsilon}{2{L}}.\]

\subsection{Proof of \thmref{thm:main}}
Fix a sample $S=\{\sampox{i}\}_{i=1}^m$ and $\gamma\ge \frac{\log 2L/\lambda \epsilon}{\lambda}$.
 Let \[\mathcal{L}(\vv) = \mathbb{E}(\ell(\vv^\top \phi_{\gamma}(\xx_\oo),y) \quad  \hat{\mathcal{L}}(\vv) =\frac{1}{m} \sum_{i=1}^m \ell(\vv^\top \phi_{\gamma}(\sampox{i}),\] the expected and empirical losses of the vector $\vv$.

Further denote by 
\[ F_c (\vv)= \frac{1}{2 C} \|\vv\|^2 + \mathcal{L}(\vv) \quad \hat{F}_c (\vv)= \frac{1}{2 C} \|\vv\|^2 + \hat{\mathcal{L}}(\vv)\]
Let $C(m) \in O\left(\sqrt{\frac{m}{\log 1/\delta}}\right)$.
Run \algref{alg:GDMD} with $T=m$ and let $\bar{\vv}=\frac{1}{T} \sum_{i=1}^T \vv_t$. By \thmref{thm:main2}, item \ref{thm2:optimal} we get:
\[ \hat{F}_{C(m)}(\bar{\vv}) \le  \min \hat{F}_{C(m)}(\vv)+ O\left(\frac{C(m) L^2\Gamma(\epsilon)}{m}\right)\] 

Note that $\|\phi_{\gamma}(\xx_\oo)\|\le \sqrt{\frac{d^{\gamma}-1}{d-1}} \|P_{\oo}\xx\|\le  \sqrt{\Gamma(\epsilon)}$. We now apply Corollary 4. in \cite{sridharan2009fast} with $B=\sqrt{\Gamma(\epsilon)}$ to obtain the following bound  (with probability $1-\delta$) for every $\ww$:

\[ \mathcal{L}(\bar{\vv}) \le \mathcal{L}(\ww) + O\left(\sqrt{\frac{L^2\Gamma(\epsilon)\|\ww\|^2 \log(1/\delta)}{m}}\right)\]
In particular for every $\|\ww\|\le \sqrt{\Gamma(\epsilon)}$ we have
\[ \mathcal{L}(\bar{\vv}) \le \mathcal{L}(\ww) + O\left(\sqrt{\frac{L^2\Gamma(\epsilon)^2\log(1/\delta)}{m}}\right).\]
From \thmref{thm:main2}, item \ref{thm2:competetive} we have that for some $\|\ww\|\le \sqrt{\Gamma(\epsilon)}$:

\[\mathcal{L}(\ww) \le \min_{\|\ww\|\le 1} \mathbb{E}(\ell(\ww^\top \xx,y)+\epsilon.\]

The result now follows from the choice of $m$.
\subsection{Proof of Theorem \ref{thm:regret}}

Before proving the theorem, we formally define the sequences for which the algorithm applies:  a $\lambda$-regular sequence is one such that the uniform distribution over the sequence elements is $\lambda$-regular  with associated subspace $E$. 

\ignore{
\begin{definition}[$\lambda$-regular-sequence]\label{def:seq-regularity}
	We say that a sequence $\{(\xx^t,\oo^t,y_t ), t=1,2,...,T\}$  is $\lambda$-regular with associated subspace $E$  if the following hold for all $t \in [T]$:
	\begin{enumerate}
		\item\label{as:compactness2} $\|P_{\oo^t} \xx^t\|\le 1$.
		\item\label{as:reconstructible2} For every $\uu$, $P_{\oo^t}P_E\uu=0$ if and only if $P_E\uu=0$.
		\item\label{as:nonsingular2} If $\lambda_{\oo^t}>0$ is a strictly positive singular value of the matrix $P_{\oo^t}P_E$ then $\lambda_{\oo^t} \ge \lambda$.
	\end{enumerate}
\end{definition}
}

\begin{proof}[Proof of Theorem \ref{thm:regret}]
Let $E^*$ denote the adversarially chosen subspace and $Q^*$ The projection associated with it. Since the sequence $\{(\xx^t,\oo^t,y_t )$ is $\lambda$-regular w.r.t. subspace $E^*$, we have by Lemma \ref{lem:approximation},
\[ \forall \|\ww\|\leq 1 \ . \ \|f^\gamma_{\ww, I-Q^*}(\xx_\oo) - f_{\ww,Q^* }(\xx_\oo)\| \le \frac{(1-\lambda)^{\gamma}}{\lambda} \leq \frac{1}{\lambda} e^{- \lambda \gamma}  \]
Thus, taking $f^{ \gamma } _{\ww^*,I-Q^*} \in \mathcal{F}^\gamma$ we have
\begin{eqnarray*}
& \min_{ \ww \in B_d } \sum_t \ell(f_{\ww,Q^*} (\xx^t_{\oo_t}) ,y_t)  -     \sum_t \ell(f^{ \gamma } _{\ww^*,I-Q^*} (\xx^t_{\oo_t}) ,y_t) \\
&  = \sum_t \ell(f^*_{\ww,Q} (\xx^t_{\oo_t}) ,y_t)  -      \sum_t \ell(f^{ \gamma } _{\ww^*,I-Q^*} (\xx^t_{\oo_t}) ,y_t) \\
& \leq  \sum_t L \| f^*_{\ww,Q} (\xx^t_{\oo_t}) - f^{\gamma, *}_{\ww,Q} (\xx^t_{\oo_t}) \|   & \mbox{ $\ell$ is $L$-Lipschitz } \\
& \leq TL \frac{1}{\lambda} e^{- \lambda \gamma}  & \mbox { Lemma \ref{lem:approximation} } 
\end{eqnarray*}
Hence it suffices to show that 
\begin{eqnarray*}
& \sum_t \ell( \vv_t ^\top \phi_\gamma(\xx^t_{\oo_t}) , y_t ) -    \sum_t \ell(f^{ \gamma } _{\ww^*,I-Q^*} (\xx^t_{\oo_t}) ,y_t) \\
&    \leq \sum_t \ell( \vv_t ^\top \phi_\gamma(\xx^t_{\oo_t}) , y_t ) -   \min_{f_{\ww,Q} \in \mathcal{F}^\gamma} \sum_t \ell(f^\gamma_{\ww,Q} (\xx^t_{\oo_t}) ,y_t) = O(\sqrt{T} ) 
\end{eqnarray*}

Corollary \ref{cor:improper} asserts that
$$ f^\gamma_{\ww,Q}(\xx_\oo) = \vv\cdot \phi_\gamma(\xx_\oo) $$

Thus, 	the theorem statement can be further reduced to
\begin{equation} \label{eqn:thm31}
\sum_t \ell( \vv_t ^\top \phi_\gamma(\xx^t_{\oo_t} ) , y_t ) -   \min_{\vv_* \in B_\Gamma({\Gamma})  } \sum_t \ell(\vv_* ^\top \phi_\gamma(\xx^t_{\oo_t}) , y_t )  = O(\sqrt{T}) 
\end{equation}

We proceed to prove equation \eqref{eqn:thm31} above. 

Algorithm \ref{alg:GDMD} applies the following update rule
$$ \vv_{t+1} =   \sum_{i=1}^t \alpha_i^{(t)}\phi_\gamma(\xx^i_{\oo^i})  $$ 	 
where $\ww_{t+1}$ can be re-written as:
\begin{eqnarray} 
\vv_{t+1} & = (1- \eta_t \rho  ) \vv_{t}  -  \eta_t \ell'(\vv_t^\top \phi_\gamma(\xx_{\oo_t}^t)) \phi_\gamma(\xx^t_{\oo^t}) \notag \\
&  = \vv_{t}  -  \eta_t \nabla \tilde{\ell}_t(\vv_t)  \label{eqn:shalom2} 
\end{eqnarray}
where
$$ \tilde{\ell}_t(\vv) = \ell(\vv^\top \phi_\gamma(\xx_{\oo_t}^t))  + \frac{\rho}{2} \|\vv\|^2 $$
The above implies a bound on the norm of the gradients of $\tilde{\ell}_t$, as given by the following lemma:
\begin{lemma} \label{lemma:shalom1}
For all iterations $t \in [T]$ we have
$$ \| \vv_t\| \leq LX \sqrt{\Gamma} \ , \ \| \nabla \tilde{\ell}_t (\vv_t) \| \leq 2 LX \sqrt{\Gamma}$$ 
\end{lemma}

Equation \eqref{eqn:shalom2} implies that KARMA applies the online gradient descent algorithm on the functions $\tilde{\ell}$ which are ${\rho}$-strongly-convex. 
Hence, the bound of  Theorem 3.3 in \cite{ocobook}, with appropriate learning rates $\eta_t$ and with $\alpha  = \rho$, $G = 2 LX \sqrt{\Gamma}$) by lemma \ref{lemma:shalom1}, gives  
$$ \sum_t \tilde{\ell}_t(\vv_t) - \min_{\vv^* } \sum_t \tilde{\ell}_t( \vv^*) \leq \frac{2 L^2 X^2 \Gamma} {\rho}(1 + \log T)  $$ 	

This directly implies our theorem since (recall that $\|\vv^*\| \leq B$ by assumption):
\begin{eqnarray*}
& \sum_t \ell( \vv_t ^\top \phi_\gamma(\xx^t_{\oo_t}) , y_t ) -   \min_{ \|\ww\|\leq 1 } \sum_t \ell(f_{\ww,Q^*} (\xx^t_{\oo_t}) ,y_t) \\
& = \sum_t \tilde{\ell}_t(\vv_t) - \min_{\vv^* } \sum_t \tilde{\ell}_t( \vv^*) + \frac{\rho}{2} ( \sum_t \| \vv^* \|^2 - \|\vv_t \|^2  )\\
& \leq \frac{2 L^2 X^2 \Gamma} {\rho}(1 + \log T) + \frac{\rho}{2} T \cdot B 
\end{eqnarray*}


\end{proof}

\begin{proof}[Proof of Lemma \ref{lemma:shalom1}]
First, notice that the norms of the gradients of the loss functions $\ell$ can be bounded by 
\begin{eqnarray*}
	\| \nabla \ell ( \vv_t ^\top \phi_\gamma(\xx^t_{\oo_t} ) , y_t ) \|  =  | \ell'( \vv_t^\top \phi_\gamma(\xx^t_{\oo^t}) , y_t) | \cdot \| \phi_\gamma(\xx^t_{\oo^t}) \| \leq L X \sqrt{ \Gamma} 
\end{eqnarray*}
where the last inequality follows from the Lipschitz property of $\ell$ and the fact that $\phi_\gamma(\xx^t_{\oo^t})$ is a vector in $\reals^\Gamma$, with coordinates from the vector $\xx^t$, and the bound  $\| \xx^t\|_\infty \leq X$.

Next, we prove by induction that $\| \vv_t\| \leq L X \sqrt{\Gamma}$.  
For $t=0$ we have $\vv_1 = 0$. 
Equation \eqref{eqn:shalom2} implies that $\vv_{t+1}$ is a convex combination of two vectors:
\begin{eqnarray*}
\| \vv_{t+1} \| & = \| (1- \eta_t C ) \vv_{t}  -  \eta_t \ell'(\vv_t^\top \phi_\gamma(\xx_{\oo_t}^t)) \phi_\gamma(\xx^t_{\oo^t})\| \\
& \leq \max \left\{  C \| \vv_{t} \| \ , \ \| \nabla \ell( \vv_t^\top \phi_\gamma(\xx_{\oo_t}^t) ) \|  \right\} \\
& \leq \max \left\{  C LX \sqrt{\Gamma}  \ , \ \| \nabla \ell( \vv_t^\top \phi_\gamma(\xx_{\oo_t}^t) ) \|  \right\} & \mbox{induction hypothesis}\\
& \leq \max \left\{  C LX \sqrt{\Gamma}  \ , \ LX \sqrt{\Gamma} \right\} & \mbox{above bound on $\nabla \ell$ }\\
& \leq LX \sqrt{\Gamma} & \mbox{$C < 1 $}
\end{eqnarray*}
We can now conclude with the lemma,  by definition of $\tilde{\ell}_t$
$$ \| \nabla \tilde{\ell}_t(\vv_t) \| \leq \| \nabla \ell(\vv_t^\top \phi_\gamma(\xx_{\oo_t}^t) ) \| + \frac{C}{2} \| \vv_t \| \leq LX \sqrt{\Gamma} + \frac{C}{2} LX \sqrt{\Gamma} \leq 2 LX \sqrt{\Gamma} $$ 

\end{proof}

\end{document}